\documentclass{article}

\usepackage{microtype}
\usepackage{graphicx}
\usepackage{amssymb}
\usepackage{amsmath}
\usepackage{subfigure}
\usepackage{amsthm}
\usepackage{multirow}
\usepackage{comment}

\usepackage[linesnumbered, ruled,algo2e]{algorithm2e}

\newtheorem{theorem}{Theorem}
\newtheorem{lemma}{Lemma}
\newcommand{\subvalue}[2]{U_{#1}^{#2}}
\newcommand{\subvalueonly}{U}

\usepackage{booktabs} 

\usepackage{hyperref}

\usepackage[accepted]{icml2019}

\icmltitlerunning{Learning Independently-Obtainable Reward Functions}

\begin{document}

\twocolumn[
\icmltitle{Learning Independently-Obtainable Reward Functions}



\icmlsetsymbol{equal}{*}

\begin{icmlauthorlist}
\icmlauthor{Christopher Grimm}{umich}
\icmlauthor{Satinder Singh}{umich}
\end{icmlauthorlist}

\icmlaffiliation{umich}{University of Michigan
Ann Arbor, MI 48109}

\icmlcorrespondingauthor{Christopher Grimm}{crgrimm@umich.edu}
\icmlcorrespondingauthor{Satinder Singh}{baveja@umich.edu}

\icmlkeywords{Machine Learning, ICML}

\vskip 0.3in
]



\printAffiliationsAndNotice{}  

\begin{abstract}
We present a novel method for learning a set of disentangled reward functions that sum to the original environment reward and are \textit{constrained to be independently obtainable}. We define independent obtainability in terms of value functions with respect to obtaining one learned reward while pursuing another learned reward. 
Empirically, we illustrate that our method can learn meaningful reward decompositions in a variety of domains and that these decompositions exhibit some form of generalization performance when the environment's reward is modified. Theoretically, we derive results about the effect of maximizing our method's objective on the resulting reward functions and their corresponding optimal policies. 
\end{abstract}

\section{Introduction \& Related Work}


We introduce a novel method for discovering disentangled structure in the reinforcement learning (RL) setting
by learning a decomposition of the agent's reward function,
such that the pursuit of one decomposed reward does not result in the collection of another, i.e., the individual rewards in the decomposition are \textbf{independently-obtainable\/}. In the following sections we briefly discuss related work, detail our method, and empirically and theoretically explore the resulting structure learned from the reward decomposition.

\paragraph{Related Work}
When viewed through the lens of RL, methods for learning disentangled representations can be classified according to how they utilize their learned representations. 

Some methods seek robust interpretable disentangled features \cite{NIPS2018_8193,NIPS2016_6051,pmlr-v80-yingzhen18a,tran2017disentangled}. For example, \citet{higgins2016beta} does so by creating an ``information bottleneck'' \cite{tishby2015deep} that pressures the latent representation to be unit Gaussian. \citet{chen2016infogan} accomplishes a similar goal by maximizing the mutual information between components of their latent representation and other independent random variables. 
Methods such as these have been leveraged in RL to decompose the state of the environment. In particular, \citet{laversanne2018curiosity} have applied $\beta$-VAE to learn disentangled, ``modular'' representations of the environment state for use in many goals RL~\cite{kaelbling1993learning}. 

However, in \citet{laversanne2018curiosity}, the decomposition of environment state and learning corresponding control policies are treated as separate processes. There has been other work that explores jointly learning state decompositions and corresponding policies. \citet{thomas2017independently} defines an alternative notion of disentanglement: ``independent controllability'' which pairs together components of the learned state representation with control policies and measures the degree to which policies can control their corresponding components independently of other components. 

While \citet{laversanne2018curiosity} and \citet{thomas2017independently} each leverage some notion of disentanglement to address RL problems, their methods do not take into account the reward function. This motivates our exploration into directly decomposing the reward function of the environment (of course, rewards are functions of state and have corresponding policies and hence decomposing rewards also implicitly decomposes states as well as policies). 

In addition to those that decompose parts of the environment, many methods exist that exploit existing decompositions. \citet{guestrin2002multiagent} and \citet{kok2006collaborative} rely on existing state factorizations to efficiently coordinate agents in the multi-agent learning setting \cite{hu1998multiagent}; \citet{van2017hybrid} and \citet{russell2003q} propose methods of learning given an existing reward factorization. Such works are complementary to our own.

\section{RL Framework \& Reward Decomposition}

\paragraph{RL Framework} We consider RL problems formulated as Markov Decision Processes (\textbf{MDPs}; \citealt{sutton1998reinforcement}), which we represent here by the tuples: $\langle \mathcal{S}, \mathcal{A}, \mathcal{R}, \mathcal{T}, \gamma
\rangle$, 
where $\mathcal{S}$ is a set of states that the environment can be in, $\mathcal{A}$ is a set of actions that an agent operating in the environment can perform at any time, $\mathcal{R} : \mathcal{S} \mapsto \mathbb{R}$ is a function mapping environment states to their corresponding rewards, $\mathcal{T} : \mathcal{S} \times \mathcal{A} \times \mathcal{S} \mapsto [0,1]$ represents the probability of transitioning between two particular states when an action is taken, and $0 \leq \gamma < 1$ is a discount factor that makes future reward less valuable than more immediate reward.

The goal of an RL agent is to act so as to maximize the expected discounted reward, or value, over an infinite horizon, defined as follows:
\begin{equation*}
V^\pi(s) = \mathbb{E}\left[\sum_{t=0}^\infty \gamma^t \mathcal{R}(s_t)  \ \Bigg|\ \pi, s_0 = s \right],
\end{equation*}
where the expectation is over trajectories generated by starting at state $s$ in the environment and behaving according to the policy $\pi$ (a policy maps states to actions or more generally distributions over actions). 
RL methods learn optimal policies defined
as follows:
$\pi^* = \underset{\pi}{\text{argmax}}\  V^\pi.$

\paragraph{Reward Decomposition} The central aim of our method is to learn an \emph{additive} factorization (henceforth, decomposition) of the reward function $\mathcal{R}$ of the following form:
\begin{equation}
\mathcal{R}(s) = \sum_{i=1}^n \mathcal{R}_i(s)
\label{eq:reward_factorization}
\end{equation} 
where each $\mathcal{R}_i(s) : \mathcal{S} \mapsto \mathbb{R}$ can be thought of as a type of sub-reward (below we add constraints that make these sub-rewards independently obtainable).

Given such a decomposition, 
we get a Factored Markov Decision Process (\textbf{fMDP}; \citealt{degris2013factored}) as follows: $\langle \mathcal{S}, \mathcal{A}, \mathcal{R}_1, \ldots, \mathcal{R}_n, \mathcal{T}, \gamma \rangle$. 
Value functions and policies with respect to pursuing individual reward functions in the fMDP can be defined as follows:
\begin{equation}
\begin{aligned}
&\subvalue{i}{\pi}(s) = \mathbb{E}\left[ \sum_{t=0}^\infty \gamma^t \mathcal{R}_i(s_t)  \ \Bigg|\  \pi, s_0 = s \right],  \, \mbox{and} \,\\
&\pi^*_i = \underset{\pi}{\text{argmax}} \ \subvalue{i}{\pi}.
\end{aligned}
\end{equation}\label{eq:factor_value_function}
\textit{Note} that throughout we will use $\{\subvalueonly_i\}$ to denote value functions for the \emph{learned} rewards $\{\mathcal{R}_i\}$ and $V$ for the value functions for the \emph{environment} reward $\mathcal{R}$.

\paragraph{Motivating and Defining Disentangled Reward Decompositions}
Notice that for any MDP there exist an endless array of different decompositions, many of them ``uninteresting''. 
For example, the following two decompositions:
\begin{equation*}
\begin{aligned}
&\mathcal{R}_i(s) = \frac{1}{n} \mathcal{R}(s) \ \ \forall i \in [n], \, \mbox{and} \\ &\mathcal{R}_i(s) = \mathcal{R}(s) \text{ if } i = 1 \text{ else } 0,
\end{aligned}
\end{equation*}
are valid in that they sum to the environment reward function, but uninteresting in that they encode no additional information about the environment not present in the unfactored reward function. 

In this paper, we propose that one approach to encouraging ``interesting'' reward decompositions is for them to satisfy the following two desiderata as best possible:
\begin{enumerate}
\item Each reward should be able to be obtained independently of other rewards. (i.e., the policy that optimally obtains $\mathcal{R}_i$ should not obtain $\mathcal{R}_j$ for any $i \neq j$). 

\item Each reward should be non-trivial. (an example of a trivial reward 
is of the form $\mathcal{R}_i(s) = 0 \ \ \forall s \in \mathcal{S}$).
\end{enumerate}
We can roughly codify these properties using the factor-specific value functions defined in the following equations:
\begin{equation}
J_\text{independent}(\mathcal{R}_1, \ldots, \mathcal{R}_n)  = \mathbb{E}_{s \sim \mu} \left[ \sum_{i \neq j} \alpha_{i,j}(s) \subvalue{i}{\pi_j^*}(s)\right]
\label{eq:J_indep}
\end{equation}
\begin{equation}
J_\text{nontrivial}(\mathcal{R}_1, \ldots, \mathcal{R}_n) = \mathbb{E}_{s \sim \mu} \left[ \sum_{i=1}^n \alpha_{i,i}(s) \subvalue{i}{\pi_i^*}(s) \right]
\label{eq:J_nontriv}
\end{equation}
where $\mu$ is some distribution with support on $\mathcal{S}$  and $\alpha_{i,j}(s)$ are functions that can be used to control the weighting of different value function terms in $J_\text{independent}$ and $J_\text{nontrivial}$. For ease of exposition we set $\alpha_{i,j}(s) = 1$ for all $i,j \in [n]$ and $s \in \mathcal{S}$ in subsequent equations (for our choice of $\alpha_{i,j}$ in our experiments, see Section~\ref{sec:illustrativeresults}) .

Intuitively, the above two desiderata make it possible to view each reward function as a different ``resource'' that an agent can collect in the world. 
$J_\text{independent}$ then encodes the degree to which the $i$'th resource (reward) is collected when the agent is attempting to collect the $j (\neq i)$'th resource (reward). In line with our first desideratum, we should expect ``interesting'' reward decompositions to have small values of $J_\text{independent}$. 
Similarly, 
large values of $J_\text{nontrivial}$ should ensure that each of the factors encodes something about the environment: that they are non-trivial under the second desideratum. 

Capturing the two desiderata, we define one reward decomposition as more disentangled than another if it has a \emph{larger} value of 
\begin{equation*}
J_\text{disentangled} = J_\text{nontrivial} - J_\text{independent}.
\end{equation*}
Next we present our method for learning reward decompositions in MDPs by maximizing $J_\text{disentangled}$.


\section{Proposed Method}

We use a parameterized \emph{reward decomposition network} (neural network with parameters $\theta$) that learns a function $F^\theta : \mathcal{S} \mapsto \mathbb{R}^n$. The outputs of this function are used to define a reward decomposition through a softmax function
\begin{equation} \label{eq:softmax_parameterization}
    R^\theta_i(s) = \mathcal{R}(s) \cdot \frac{\exp(F^\theta_i(s))}{\sum_{j=1}^n \exp(F^\theta_j(s))},
\end{equation}
where $\mathcal{R}$ is the environment reward function. The vector of decomposed rewards is thus
$R^\theta(s) = (R_1^\theta(s), \ldots, R_n^\theta(s))$.
See Figure~\ref{fig:architecture}a for a visualization.

Note that $\theta$ defines the reward decomposition which in turn defines the matrix of value functions used in the definition of $J_\text{disentangled}$ and so for ease of notation we can refer to 
$J_\text{disentangled}(R_1^\theta, \ldots, R_n^\theta)$ as just $J_\text{disentangled}(\theta)$. We similarly abbreviate $J_\text{nontriv}$ and $J_\text{indep}$ when appropriate.

We use approximate gradient ascent to update the parameters $\theta$ as follows:
\begin{equation*}
\begin{aligned}
\nabla_\theta J_\text{disentangled}(\theta) &= \nabla_\theta\left(J_\text{nontriv} - J_\text{indep}\right) \\
&= \mathbb{E}\left[\sum_{i=1}^n \nabla_\theta \subvalue{i}{\pi_i^*}(s) - \sum_{i \neq j} \nabla_\theta \subvalue{i}{\pi_j^*}(s)\right]
\end{aligned}
\end{equation*}
\begin{equation}
\begin{aligned}
\nabla_\theta \subvalue{i}{\pi_j^*}(s) &= \nabla_\theta \mathbb{E} \left[ \sum_{t=0}^\infty \gamma^t R_i^\theta(s_t) \  \Bigg|\  \pi_j^*, s_0 = s\right] \\
&\approx \mathbb{E}\left[ \sum_{t=0}^T \nabla_\theta \gamma^t R_i^\theta(s_t) \ \Bigg|\ \pi_j^*, s_0 = s \right],
\end{aligned}
\label{eqn:gradient}
\end{equation}
where $T$ is a cutoff on the number of time-steps that a trajectory is rolled out for and $\nabla_\theta R_i^\theta(s_t)$ depends on the details of the form of the $F^\theta$ functions (which will be neural networks in our empirical work below).

Notice that computing the gradient of the disentanglement objective with respect to $\theta$ requires learning the optimal policies for each factor of the reward decomposition after each change to $\theta$.  In practice, for sample efficiency, we do incremental updates of the policies and value functions as we adapt $\theta$. We use Deep Q-Networks (\textbf{DQNs}; \citealt{mnih2015human}) to learn optimal policies for the decomposed rewards (collectively called the \emph{Policy Networks}; see Figure~\ref{fig:architecture}b). Additionally, in order to compute $J_\text{disentangled}$ we must be able to collect multiple trajectories from the same starting state (see Equation~\ref{eqn:gradient}). This requires that our environments be resettable to specific states (more details for our specific implementations are in Algorithm~\ref{alg:low_level} presented in Appendix~\ref{app:algorithm}). 

\begin{figure}
\centering
\begin{subfigure}
\centering
\includegraphics[scale=0.2]{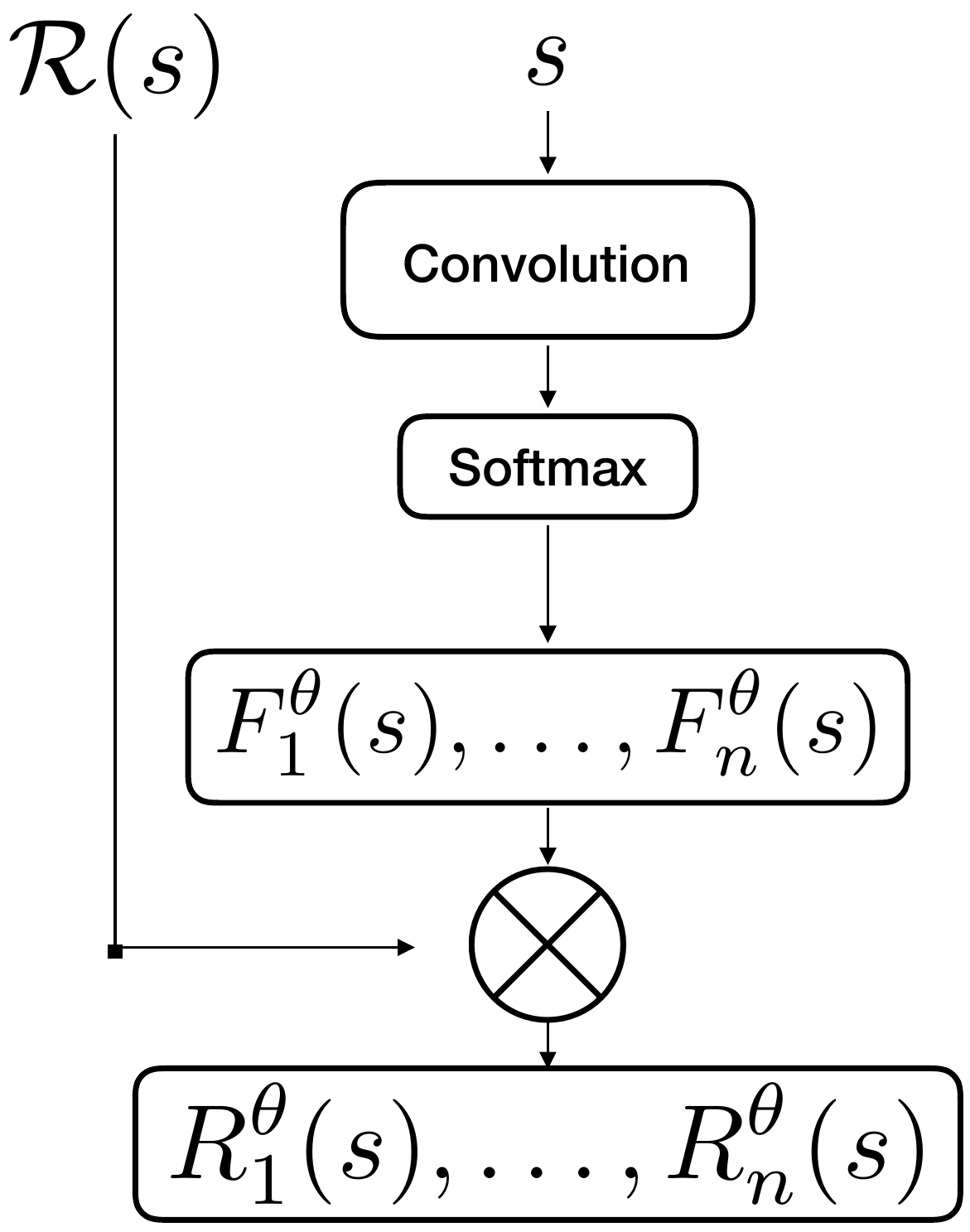}
\end{subfigure}
~
\begin{subfigure}
\centering
\includegraphics[scale=0.2]{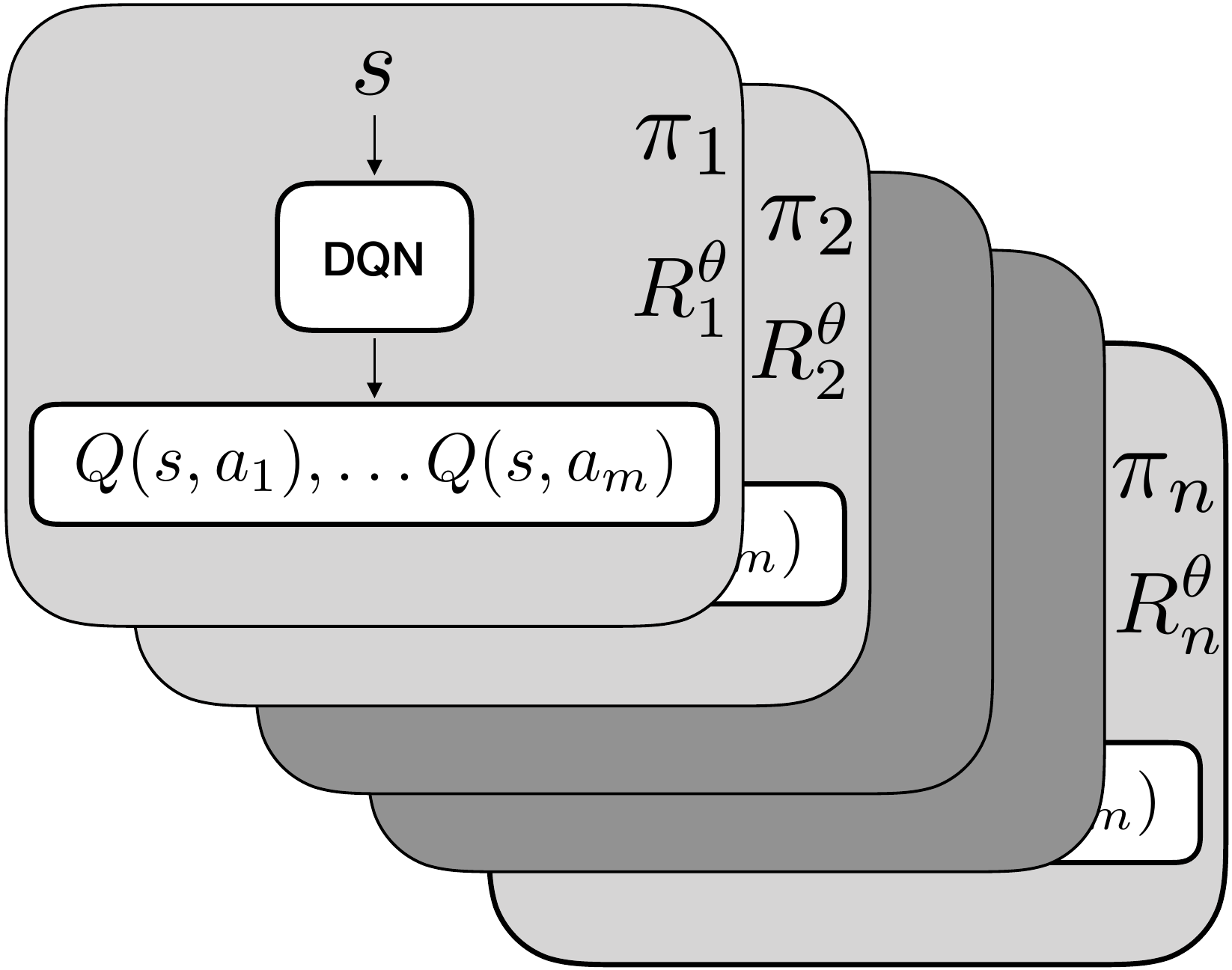}
\end{subfigure}
\caption{ A representation of the architecture used in Algorithm~\ref{alg:low_level} presented in Appendix~\ref{app:algorithm}. Left: our ``reward decomposition network'' which takes a state $s$ and produces $F^\theta  = (F_1^\theta(s), \ldots, F_n^\theta(s)))$. The environment reward is then applied as in Equation~\ref{eq:softmax_parameterization} to produce the decomposed rewards $(R_1^\theta(s), \ldots, R_n^\theta(s))$. Right: collection of DQNs: $Q^{\phi_1}, \ldots Q^{\phi_n}$ that learn policies for each decomposed reward.} \label{fig:architecture}
\end{figure}


\begin{figure*}
\centering
\begin{subfigure}
\centering
\includegraphics[scale=0.18]{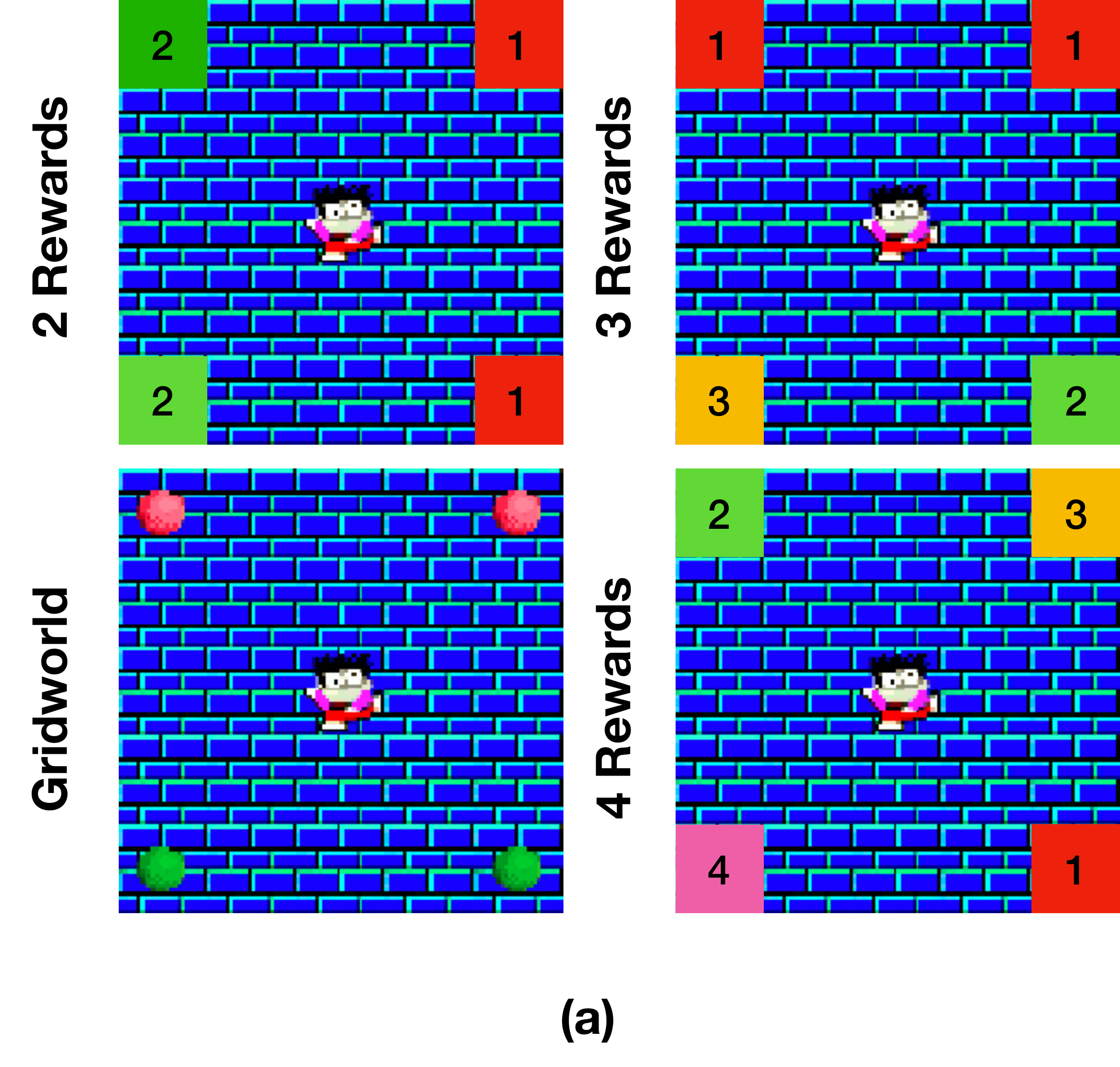}
\end{subfigure}
~
\begin{subfigure}
\centering
\includegraphics[scale=0.18]{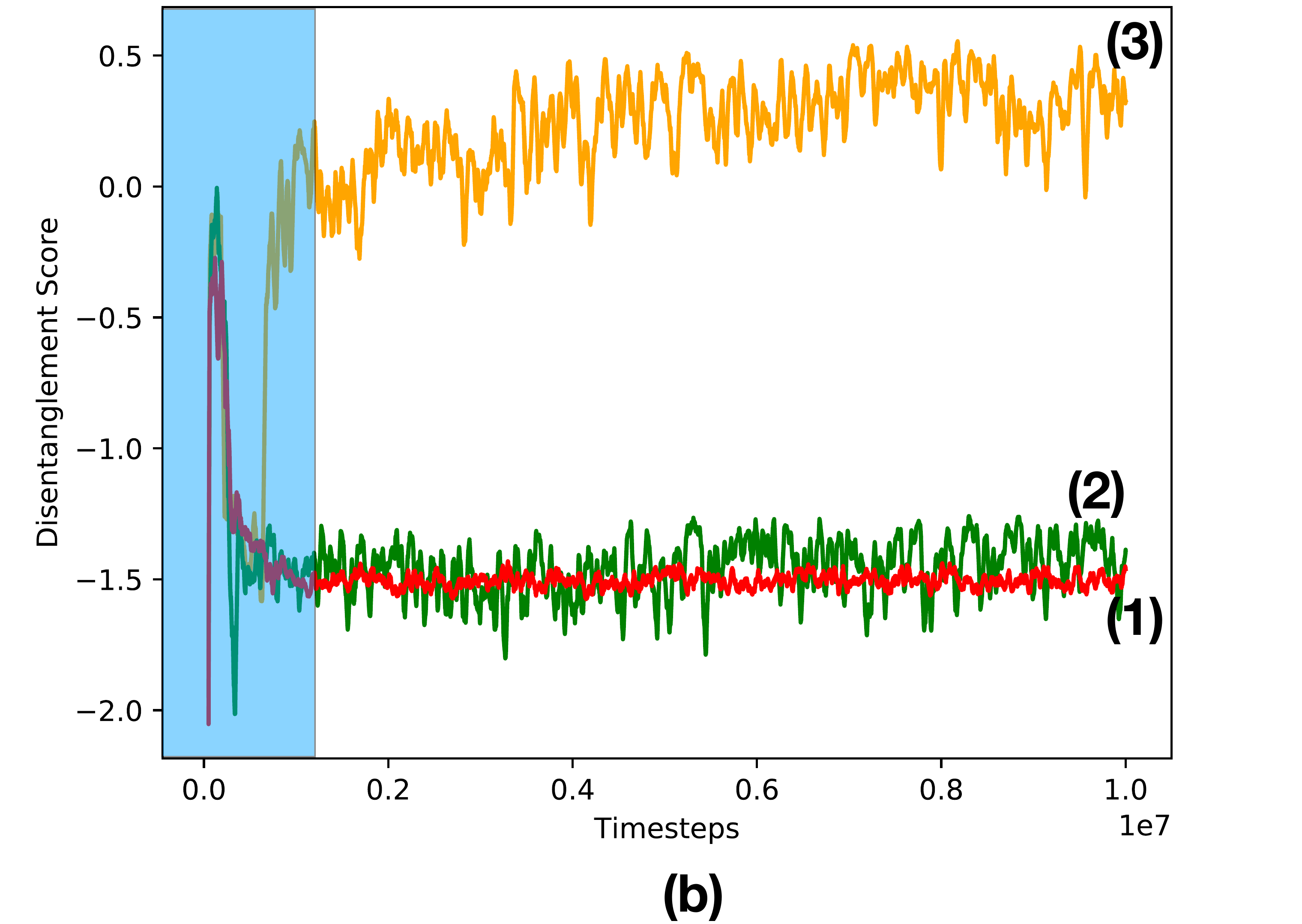}
\end{subfigure}
~
\begin{subfigure}
\centering
\includegraphics[scale=0.18]{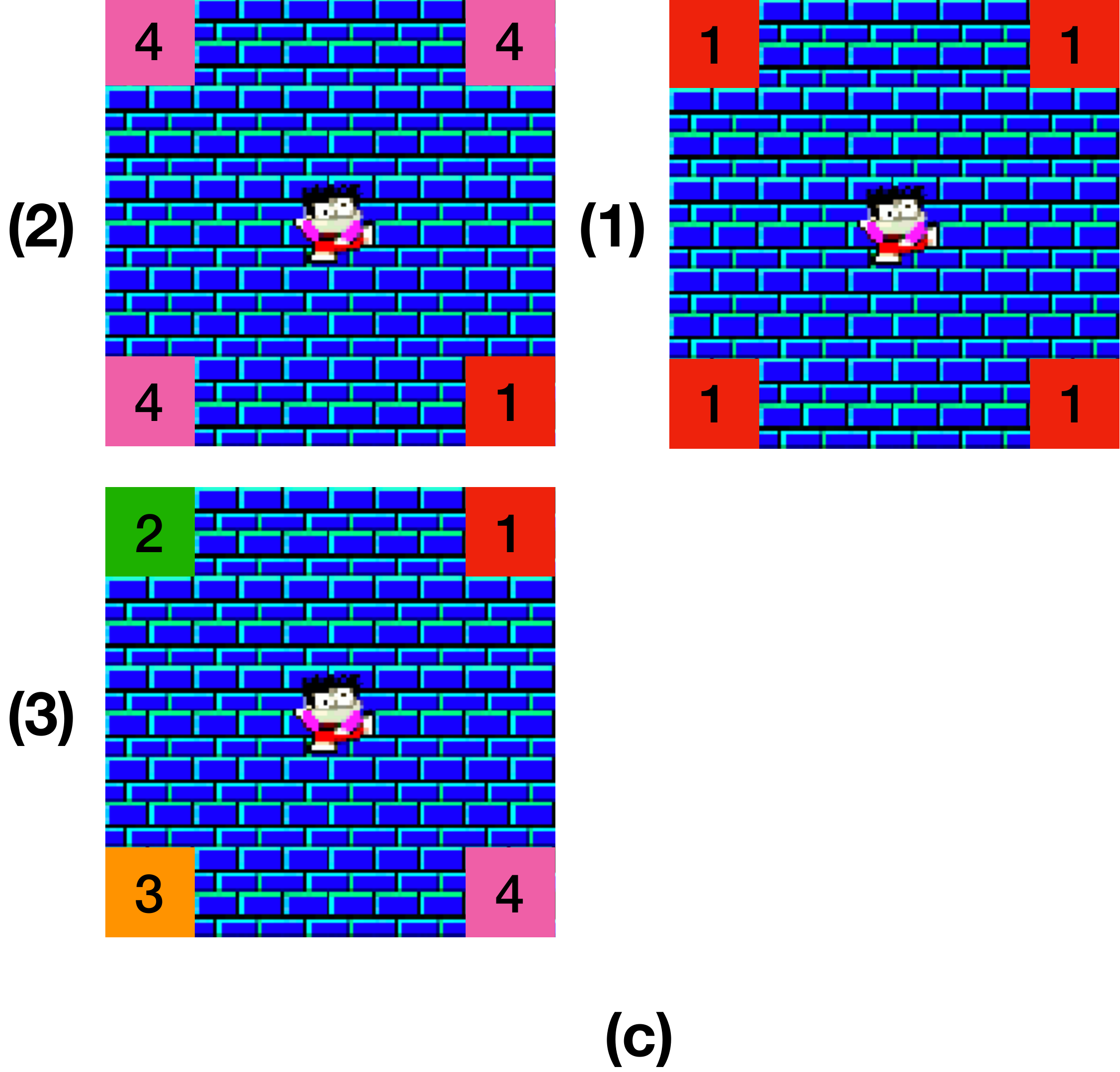}
\end{subfigure}
\caption{(a) Visualization of the gridworld environment (labeled ``Gridworld'') and reward decompositions learned as the number of learned reward functions ranges from 2 to 4. (b) Illustration of the type of instability that can occur in early training. Each curve is cherry-picked to correspond to a different asymptotic value of $J_\text{disentangled}$. (c) Corresponding reward disentanglements for each curve depicted in (b).} \label{fig:instability}
\end{figure*}

\subsection{Illustrative Results
\label{sec:illustrativeresults}}
Before we present our theoretical results and substantive empirical results, we present an empirical illustration of the kinds of reward decompositions achieved by our method in a 5x5 grid with a rewarding square at each corner (see image labeled ``Gridworld'' in Figure~\ref{fig:instability}a). The agent can move left, right, up and down. Upon reaching a rewarding square, the agent is teleported to a random square.

\paragraph{Learned Reward Decompositions}
Figure~\ref{fig:instability}a illustrates the types of reward decompositions obtained when we learn 2, 3 and 4 reward functions respectively (the numbers/colors denote which reward sources were put into the same learned reward). Notice how our disentanglement objective encourages the environment reward to be divided among the learned reward functions in a way that can be ``independently obtained'' by their associated policies. As the number of learned reward functions increases, this results in increasingly fine divisions of the environment reward: two reward functions divides the reward between halves of the environment, three reward functions separates the top half of the environment from the bottom and further divides the bottom half into bottom left and right corners and four reward functions associates each corner of the environment with a separate reward function.

\paragraph{Training Stability and Correspondence to $J_\text{disentangled}$}

Despite our disentanglement objective  discouraging the learning of trivial (i.e., zero-everywhere) reward functions, we observed that such degenerate decompositions can still emerge as a result of our training procedure. Particularly, this can occur in early training as illustrated in Figure~\ref{fig:instability}b where in the early (shaded in blue) part the $J_\text{disentangled}$ score for two runs (denoted (1) and (2)) ``drops'' to lower values, while the score for one run (denoted (3)) continues to increase with training.  The drops coincide with individual learned reward functions 
becoming trivial during early training. This can be seen in 
Figure~\ref{fig:instability}c where we show the corresponding reward decompositions. Note that run (1) puts all 4 reward sources into the same reward function while run (2) puts three reward source in one reward function and the fourth reward source in a second reward function. Only in run (3), where the $J_\text{disentangled}$ score does increase during learning do we find the disentanglement we were looking for, each reward source is in a separate reward function. It is comforting and useful that our $J_\text{disentangled}$ score is useful as an indicator of the quality of reward decompositions found.

How do we mitigate the above instability? 
As seen in Algorithm~\ref{alg:low_level} presented in the Appendix~\ref{app:algorithm}, updating the learned reward functions depends on having policies that can obtain some environment reward, and a given policy's ability to obtain environment reward depends on its learned reward function being non-trivial. 
If a learned reward function becomes trivial before the corresponding policy learns to obtain some environment reward, it can get stuck there.
We found that the prevalence of this problem could be significantly reduced with a specific choice of the $\alpha_{i,j}(s)$ coefficients defined in Equations~\ref{eq:J_indep} and \ref{eq:J_nontriv}. For all of our experiments we used 
\begin{equation*}
\alpha_{i,j}(s) = \begin{cases}
10 \cdot \frac{\exp{-2\subvalue{i}{\pi_i^*}(s)}}{\sum_{i=1}^n \exp{-2\subvalue{i}{\pi_i^*}(s)}} & i = j \\
1 & i \neq j
\end{cases} 
\end{equation*}
which corresponds to taking a softened minimum with temperature $2$ over the terms in Equation~\ref{eq:J_nontriv}. The intuition behind this choice is that by approximately maximizing the minimal $\subvalue{i}{\pi_i^*}(s)$ term, the optimization procedure is able to dynamically attend to the learned reward functions which are most in danger of becoming trivial. This choice yielded the best empirical stability among the choices we considered.

\section{Theoretical Properties} \label{sec:theoretical_properties}

In this section we detail various theoretical properties that result from the optimization of our $J_\text{disentangled}$ objective. The following theorems characterize the nature of the reward functions and their corresponding policies learned by our method.

\paragraph{Non-Overlapping Visitation Frequencies}
Our first theoretical result illustrates that the optimization of our $J_\text{disentangled}$ objective results in policies whose state-visitation frequencies have a low degree of overlap. This provides a concrete sense of the resulting disentanglement. 

Define the discounted visitation frequency of $s^\prime$ starting from $s$ under policy $\pi$ as:
\begin{equation*}
\mu_s^\pi(s^\prime) = \frac{\Psi_\pi(s, s^\prime)}{\int_{\mathcal{S}}\Psi_\pi(s,s^\prime) ds^\prime}
\end{equation*}
where $\Psi_{\pi}(s, s^\prime)$ is defined as
\begin{equation*}
\Psi_{\pi}(s, s^\prime) = \mathbb{E}\left[ \sum_{t=0}^\infty \gamma^t \boldsymbol{1}\{s_t = s^\prime\} \ \Bigg|\ s_0 = s, \pi \right] .
\end{equation*}

Specifically, we show that if two learned reward functions, $i$ and $j$, are sufficiently disentangled in that $\subvalue{i}{\pi_i^*}(s) - \subvalue{i}{\pi_j^*}(s) > C$, then the state visitation frequencies of their corresponding optimal policies must also be at least somewhat different. 
\begin{theorem} \label{thm:visit_freq}
If $\subvalue{i}{\pi_i^*}(s) - \subvalue{i}{\pi_j^*}(s) > C$ and $\mathcal{R}(s^\prime) \geq 0$ for all $s^\prime \in \mathcal{S}$ then 
\begin{equation*}
\delta(\mu_s^{\pi^*_j}, \mu_s^{\pi^*_i})  \geq \frac{(1 - \gamma) C}{2 R_\text{max}}
\end{equation*}
where $R_\text{max}$ is the maximum environment reward and $\delta(\cdot, \cdot)$ represents total variation distance.
\end{theorem}

\begin{proof}
We can alternatively represent value functions in terms of $\Psi_\pi$ as:
\begin{equation*}
\begin{aligned}
\subvalue{i}{\pi_j^*}(s) = \sum_{s^\prime \in \mathcal{S}} \Psi_{\pi^*_j}(s, s^\prime) R_i(s^\prime) ds^\prime.
\end{aligned}
\end{equation*} 
Our constraints imply that 
\begin{equation*}
\begin{aligned}
C &\leq \subvalue{i}{\pi_i^*}(s) - \subvalue{i}{\pi_j^*}(s) \\
&= \sum_{s^\prime \in \mathcal{S}}R_i(s^\prime)\left( \Psi_{\pi^*_i}(s, s^\prime) - \Psi_{\pi^*_j}(s, s^\prime) \right) \\
&\leq R_\text{max} \sum_{s^\prime \in \mathcal{S}} \left| \Psi_{\pi^*_i}(s, s^\prime) - \Psi_{\pi^*_j}(s, s^\prime)\right| \\
&= 2 R_\text{max} (1 - \gamma)^{-1} \delta(\mu_s^{\pi^*_j}, \mu_s^{\pi^*_i})
\end{aligned}
\end{equation*}
as needed.
\end{proof}

\paragraph{Saturation of Softmax-Parameterization}
We refer to a reward decomposition as saturated if for each rewarding state there is exactly one reward function that is non-zero (i.e., the decomposition never splits an environment reward between two or more reward functions).

\begin{theorem} \label{thm:saturation}
For $\alpha_{i,j}(s) = 1$ for all $i, j \in [n]$ and $s \in \mathcal{S}$ and under conditions described in Appendix~\ref{app:mathematical_results}, the optimal reward decomposition under $J_\text{disentangled}$ is saturated. 
\end{theorem}



\begin{proof}
\begin{equation*}
\begin{aligned}
J_\text{nontriv}(\theta^\prime) &= \mathbb{E}_{s \sim \mu}\left[\sum_{i=1}^n \subvalue{R_i^\theta}{\pi_i^\theta}(s)\right] \\
&= \mathbb{E}_{s \sim \mu}\left[\sum_{i=1}^n \mathbb{E}\left[\sum_{t=0}^\infty \gamma^t R_i^\theta(s_t) \left| s_0 = s \right.\right]\right] \\
&= \sum_{s \in \mathcal{S}} R(s) \sum_{i=1}^n \mu_{i,\theta}^\gamma(s) F_i^\theta(s)
\end{aligned}
\end{equation*}
where $\mu_{i,\theta}^\gamma(s) = \sum_{t=0}^\infty \gamma^t \mu_{i,\theta}^t(s)$ and $\mu^0_{i,\theta}(s) = \mu(s)$ for all $i \in [n]$ and $s \in \mathcal{S}$.

Suppose we chose an alternate parameterization $\theta^\prime$ with the property that
\begin{equation*}
F_i^{\theta^\prime}(s) = \begin{cases}
1 & \text{ if } i = \min \underset{j}{\text{argmax}}\ \mu_{j,\theta}^\gamma(s) \\
0 & \text{ else.}
\end{cases}
\end{equation*}
This alternate parameterization yields 
\begin{equation*}
\mathbb{E}_{s \sim \mu}\left[\sum_{i=1}^n \subvalue{R_i^{\theta^\prime}}{\pi_i^\theta}(s)\right] \geq \mathbb{E}_{s \sim \mu}\left[\sum_{i=1}^n \subvalue{R_i^\theta}{\pi_i^\theta}(s)\right]
\end{equation*}
with strict inequality if $F_i^\theta(s) \neq F_i^{\theta^\prime}(s)$ for any pair of $i, s$. Next notice that $\pi_i^\theta$ represents the best policy to collect the $i$'th reward under parameterization $\theta$. Under the new parameterization $\theta^\prime$, there is a new optimal policy $\pi_i^{\theta^\prime}$. It thus follows that:
\begin{equation*}
\begin{aligned}
   J_\text{nontriv}(\theta^\prime) &= \mathbb{E}_{s \sim \mu}\left[\sum_{i=1}^n \subvalue{R_i^{\theta^\prime}}{\pi_i^{\theta^\prime}}(s)\right]  \\
   &\geq
   \mathbb{E}_{s \sim \mu}\left[\sum_{i=1}^n \subvalue{R_i^{\theta^\prime}}{\pi_i^\theta}(s)\right] \\
   &\geq \mathbb{E}_{s \sim \mu}\left[\sum_{i=1}^n \subvalue{R_i^\theta}{\pi_i^\theta}(s)\right] = J_\text{nontriv}(\theta).
\end{aligned}
\end{equation*}

Finally note that:
\begin{equation*}
\begin{aligned}
&J_\text{indep}(\theta^\prime) - J_\text{indep}(\theta)  \\ &\leq S(\pi_{1:n}^\theta, \pi_{1:n}^{\theta^\prime}) - (J_\text{nontriv}(\theta^\prime) - J_\text{nontriv}(\theta))
\end{aligned}
\end{equation*}
where $S(\pi_{1:n}^\theta, \pi_{1:n}^{\theta^\prime})$ is defined in Appendix~\ref{app:mathematical_results}. 

This implies that if $S(\pi_{1:n}^\theta, \pi_{1:n}^{\theta^\prime})$ is sufficiently small then $J_\text{disentangled}(\theta^\prime) \geq J_\text{disentangled}(\theta)$.
\end{proof}

\section{Experimental Results
\label{sec:experiments}}

We present three classes of experimental results here.  
\emph{First}, we demonstrate the qualitative and quantitative properties of our reward decompositions and associated policies on the Atari 2600 games: Assault, Pacman and Seaquest. We connect these qualitative properties with the theoretical properties discussed in Section~\ref{sec:theoretical_properties}. \emph{Second}, we compare the policies learned by our method to those learned by an alternative method: Independently Controllable Factors (\textbf{ICF}; \citealt{thomas2017independently}).
\emph{Third}, we illustrate that the policies optimal for our learned reward functions can be used as actions to obtain environment reward and further show that these policies exhibit some degree of generalization performance when applied to tasks with modified environment reward functions. 

\paragraph{Common Experimental Procedure}


For all experiments involving learning a decomposition, we produce $4$ decompositions and select the best run, i.e., the decomposition from the run that achieves the best disentanglement score.
Unless otherwise stated, all curves seen in the subsequent sections should be regarded as the average of four runs with different random seeds all using the best-decomposition discussed above. For further details regarding experimental procedure and hyperparameter selection, see Appendix~\ref{app:experiment_details}.

\begin{figure*}
\centering
\includegraphics[scale=0.15]{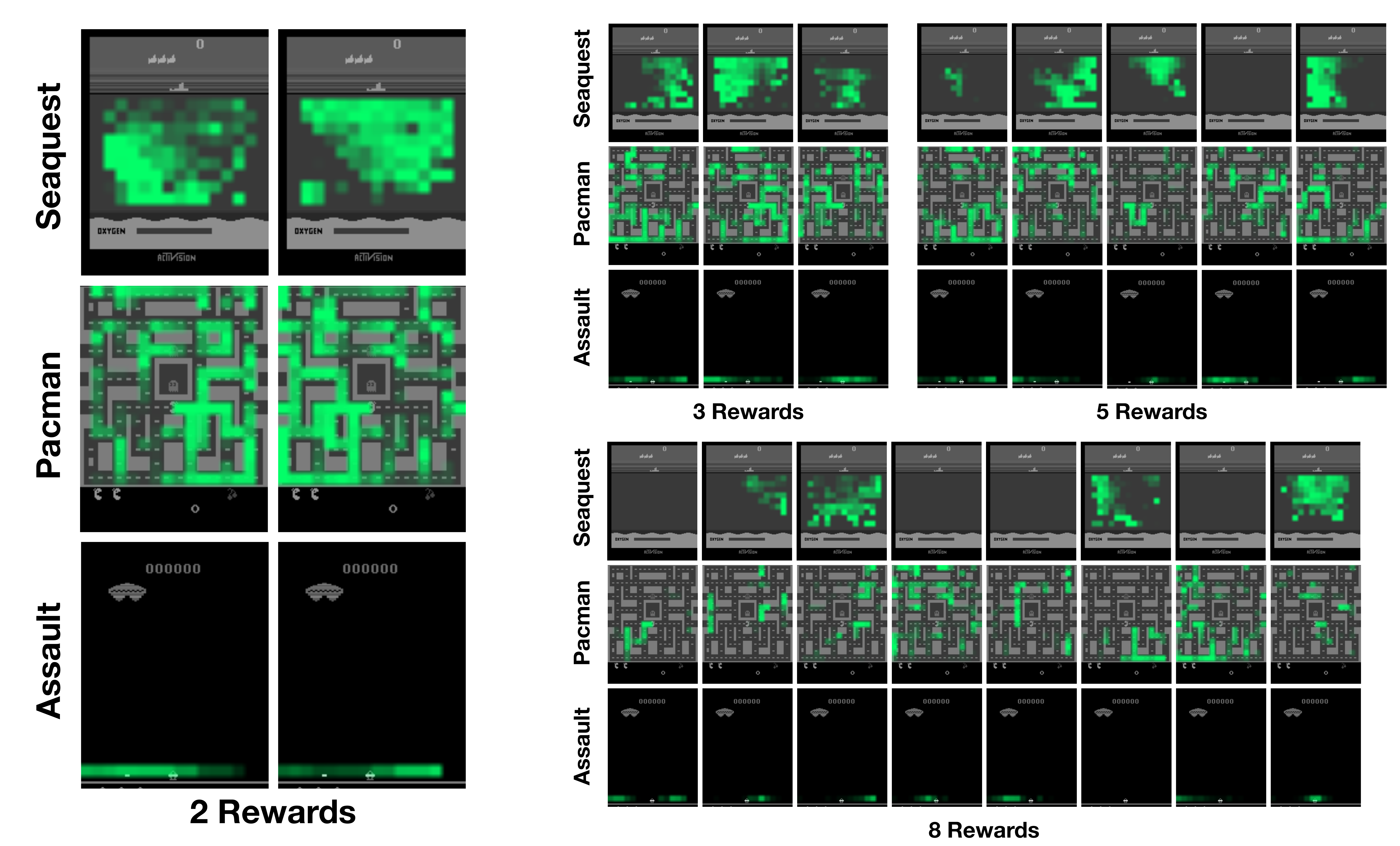}
\caption{
An illustration of the relationship between selected ``game-elements'' and different learned reward functions. The game-elements for Seaquest and Pacman are the agents that the player controls. The game-element for Assault is the horizontal location of the agent's laser. Images in this figure are arranged in three ways: instances of learning different numbers of reward functions (i.e., 2,3,5,8) are grouped together in blocks, individual Atari 2600 games are grouped together by row and within each row and block are representations of individual learned reward functions. Neon green regions in each image depict corresponding spatial regions in which the agent received a reward under a particular decomposed reward.
} \label{fig:reward_visualization}
\end{figure*}

\subsection{Qualitative \& Quantitative Analyses}


\paragraph{Reward Function Visualization Methodology}
Atari 2600 games have too many states to enumerate the learned rewards, necessitating an approximate visualization. For each game we select a ``game-element'' (see Figure~\ref{fig:reward_visualization}) and visualize its position as reward is received from the different learned reward functions. To construct our visualization, we discretize the environments spatially into $10 \times 10$ pixel bins. We then execute a random policy for $100,000$ time-steps; every time an environment reward is received we store its associated reward decomposition in the appropriate bin determined by the location of the game-element. Finally, we compute the distribution over learned reward functions at each bin and produce images for each reward function with regions of high-reward shaded in green. This is depicted in Figure~\ref{fig:reward_visualization}. 

\paragraph{Reward Function Visualization Analysis}
In Figure~\ref{fig:reward_visualization} each image shows the location of the game-element when a specific learned reward is obtained. We observe that each learned reward function generally yields reward in separate regions of the environment's state-space. This effect is most pronounced when the number of reward functions is small but persists as more are added, suggesting that the location of the game-element plays a significant role in our learned decompositions. The learning of location-dependent reward functions is consistent with Theorem \ref{thm:visit_freq} which implies that the optimization of $J_\text{disentangled}$ results in optimal policies with different occupancy frequencies. 


\paragraph{Empirical Reward Saturation}
Theorem~\ref{thm:saturation} predicts that under certain conditions, the optimal reward decomposition never splits the environment reward of a state between its learned factors. We term this property ``saturation.'' To illustrate the degree to which our method learns saturated reward functions empirically, we define the following ``saturation score:''
\begin{equation} \label{eq:saturation_percentage}
P_\text{sat}(R_1^\theta(s), \ldots, R_n^\theta(s)) = \frac{\max_i{R_i^\theta(s) / \mathcal{R}(s)} - 1/n}{1 - 1/n}    
\end{equation} 
which is defined when the environment reward is nonzero and ranges from $0$ to $1$ as the reward decomposition becomes more saturated at $s$.


Table~\ref{tbl:saturation_table} depicts the average saturation scores over all instances of positive reward obtained during the execution of a random policy across our selection of Atari 2600 games as the number of learned reward functions varies. We observe that, under Equation~\ref{eq:saturation_percentage}, our learned decompositions are extremely saturated, suggesting that Theorem~\ref{thm:saturation} holds in practice.

\begin{table}
\begin{center}
\begin{tabular}{|c||c|c|c|}
\hline
Reward Factors & Pacman & Seaquest & Assault \\
\hline
2 & .989 & .986 & .995 \\
\hline 
3 & .979 & .974 & .993 \\
\hline 
5 & .955 & .940 & .991 \\
\hline 
8 & .947 & .989 & .990 \\
\hline
\end{tabular} 
\end{center}
\caption{A table depicting the ``saturation scores'' defined in Equation~\ref{eq:saturation_percentage} for our selected Atari 2600 games as the number of reward factors varies.}\label{tbl:saturation_table}
\end{table}

\subsection{Comparison to ICF}

We now compare our method against a competing approach, ICF, which 
learns a set of latent factors as a state decomposition with the property that each factor can be controlled by a corresponding policy without changing the other factors by optimizing a ``selectivity score.''  
While ICF does not involve the reward function, we can empirically compare its resultant policies\footnote{We use the directed variant of the selectivity score proposed in \cite{thomas2017independently}. In the comparisons reported here, unless otherwise stated, we allot ICF two directed policies for every policy we let our method learn.} against those generated by our method.


\paragraph{Qualitative Gridworld Comparison}
We examine the behaviors of the policies found by ICF with 4 policies and our method with 4 reward functions on the gridworld domain in Figure~\ref{fig:sokoban_comp}. In this setting 
ICF learns policies corresponding to the cardinal directions of the domain, with each policy performing a single action (i.e., up, down, left, right) irrespective of the state. Conversely, given four reward functions, our approach learns policies that are directed to the four sources of reward. 

\begin{figure}[]
\centering 
\begin{subfigure}
\centering 
\includegraphics[scale=0.25]{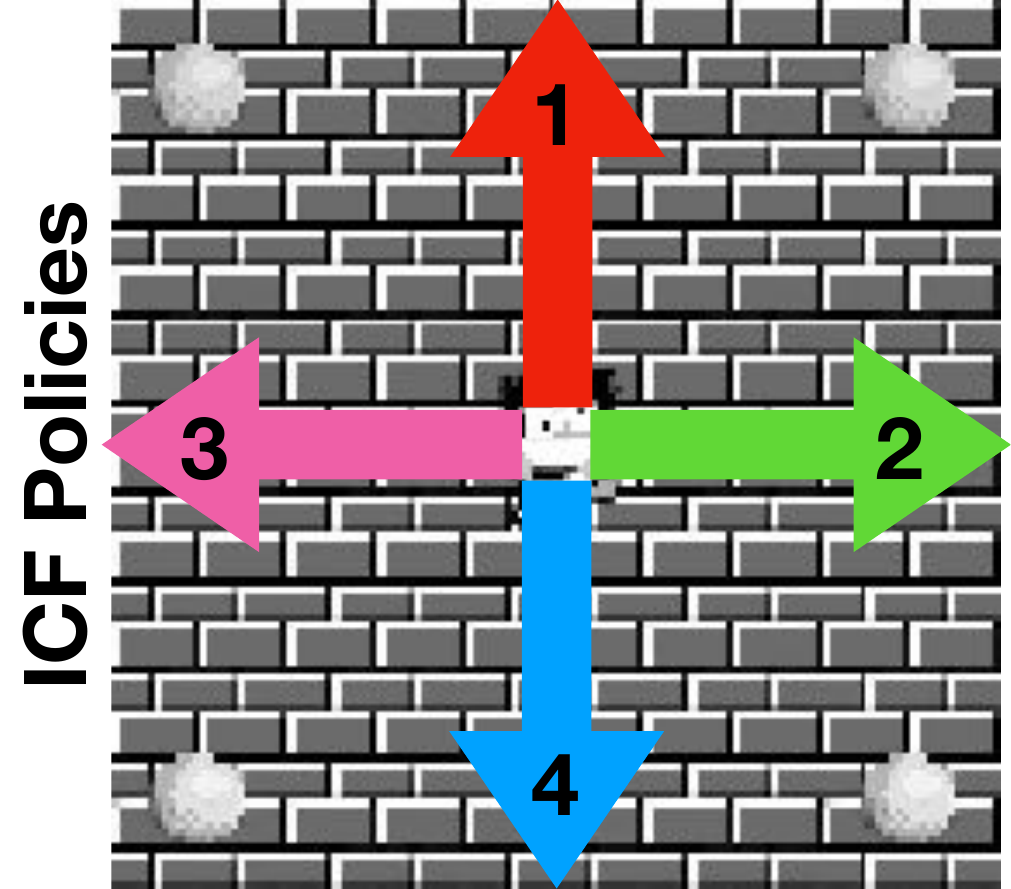}
\end{subfigure}
~
\begin{subfigure}
\centering 
\includegraphics[scale=0.25]{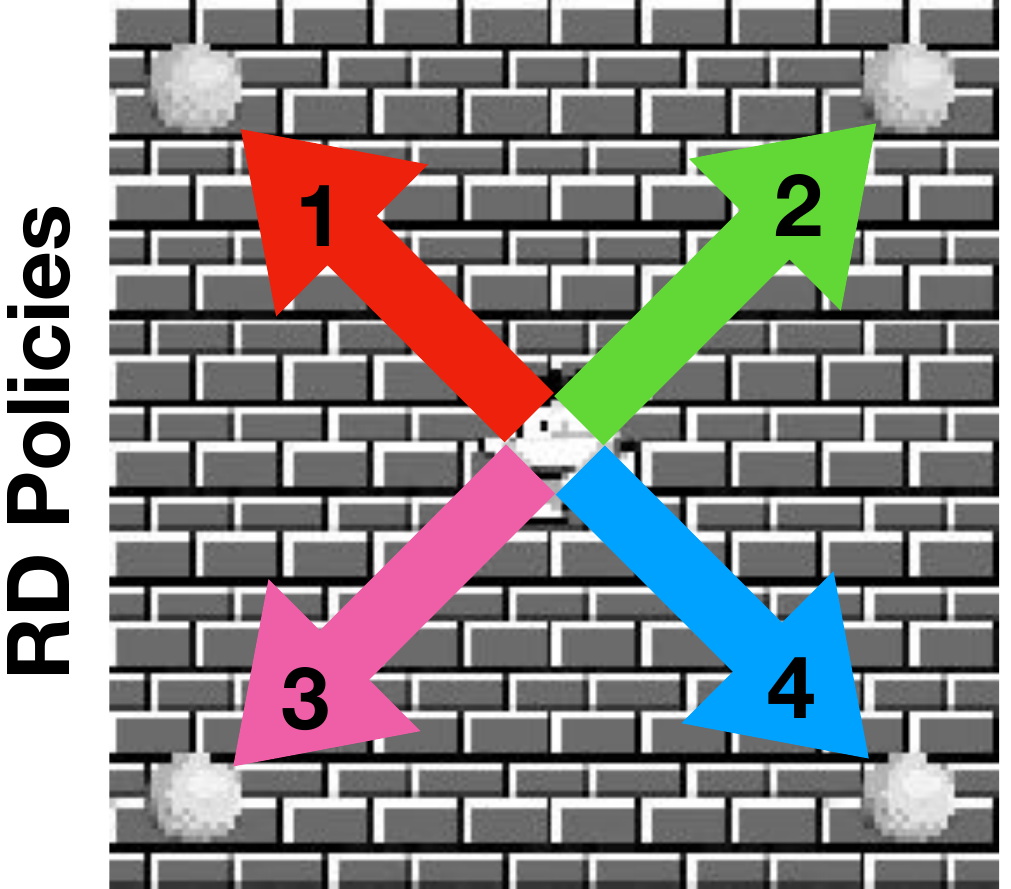}
\end{subfigure}
\caption{An illustration of the different types of policies generated by ICF (left) and our own method (right) on the gridworld domain. Notice that ICF generated policies that correspond to moving in the four cardinal directions of the game, whereas the policies generated by our method are directed toward the different sources of reward in the environment.  }\label{fig:sokoban_comp}
\end{figure}


\begin{table*}[ht]
    \centering
    \begin{tabular}{|c||c|c||c|c||c|c||c|c||c|c||c|c|}
    \hline 
        & \multicolumn{6}{c||}{Average Policy Values} &
        \multicolumn{6}{c|}{Average State-Dependence} \\ 
        \hline 
        \multirow{2}{*}{Factors} & 
        \multicolumn{2}{c||}{Pacman} &
        \multicolumn{2}{c||}{Seaquest} &
        \multicolumn{2}{c||}{Assault} & 
        \multicolumn{2}{c||}{Pacman} &
        \multicolumn{2}{c||}{Seaquest} &
        \multicolumn{2}{c|}{Assault}
        \\
         & RD & ICF & RD & ICF & RD & ICF
         & RD & ICF & RD & ICF & RD & ICF \\
         \hline \hline
         2 & 194.06 & 109.77 & 57.95 & 4.17 & 178.07 & 27.10 &
         0.161 & 0.005 & 0.185 & 0.000 & 0.278 & 0.000 \\
         \hline 
         3 & 170.75 & 92.32 & 39.09 & 9.14 & 152.20 & 30.34 &
         0.200 & 0.004 & 0.142 & 0.002 & 0.266 & 0.003 \\
         \hline 
         5 & 98.18 & 75.64 & 31.25 & 2.36 & 136.87 & 26.84 &
         0.156 & 0.004 & 0.125 & 0.001 & 0.271 & 0.012 \\
         \hline 
         8 & 65.43 & 72.15 & 21.62  & 9.22 & 78.91 & 36.42 &
         0.137 & 0.002 & 0.102  & 0.001 & 0.259 & 0.012 \\ 
         \hline 
    \end{tabular}
    \caption{Table of the average policy values and average state-dependencies (as defined in Equation~\ref{eq:state_dependence}) for both our reward decomposition method (RD) and ICF. }
    \label{tbl:policy_values}
\end{table*}

\paragraph{Quantitative Policy Comparison on Atari}
As a further comparison between ICF and our method on Atari 2600 games, we collect two statistics on the policies learned by each: the value of the environment reward averaged over all the policies, and the average degree to which each policy changes actions with respect to state. We denote this latter quantity as ``state-dependence'' and define it for a policy $\pi$ as follows:
\begin{equation} \label{eq:state_dependence}
P_\text{sdep}(\pi) = \frac{1}{|\mathcal{A}|} \sum_{a \in \mathcal{A}} \sigma_{s \sim \tau_{\pi}}(\pi(a | s))
\end{equation}
where $\tau_\pi$ denotes states sampled from a trajectory generated according to policy $\pi$ and $\sigma_{p}(\cdot)$ denotes standard deviation with respect to distribution $p$. 

We display the results of these comparisons in Table~\ref{tbl:policy_values}. Notice how the average values of the our method's policies achieve higher average environment-reward-based value than the ICF policies and \textit{substantially higher} state-dependencies. While we expect this trend in average value (since our method considers the environment reward), the near-zero state-dependencies of the ICF policies is more surprising and suggests that the policies learned by ICF are not influenced by the environment state.

\subsection{Control using Induced MDPs} \label{sec:induced_mdp} 
We also explore whether the policies generated from disentangled rewards can be useful for learning control with a DQN \cite{mnih2015human}. Replacing the actions in the original MDP with the policies found by our method ``induces'' an MDP in which selecting a policy in a state executes the action in the original MDP the policy takes in that state. The top row of Figure~\ref{fig:meta_controller_results} shows that the game score in the induced MDP rises faster but asymptotes lower than the baseline of learning in the original MDP in Seaquest and Assault. We conjecture the faster rise is because the policies learned for the disentangled rewards produced by our method are also good at obtaining environment reward.  Of course, the lower asymptote is expected because the baseline method is not limited in the behaviors that can be executed. 
We should also expect the use of the learned policies as actions to generalize to changes in the environment reward and we see this in the bottom row of Figure~\ref{fig:meta_controller_results}. The specific reward changes to each Atari game are explained in the caption of Figure~\ref{fig:restricted_rewards}. In both Seaquest and Assault, the induced agent not only learns faster than the baseline but also achieves an asymptote that is competitive with it. 


Another expected but nonetheless interesting result is that in both Seaquest and Assault performance gets better in general as we increase the number of learned rewards for both the top and bottom row of Figure~\ref{fig:meta_controller_results}b. In Pacman, the results are a bit less consistent: on the one hand in the top row learning with 3 policies as actions does a bit better than the baseline, on the other hand performance is not necessarily better with increasing numbers of policies.

\begin{figure}
\centering
\includegraphics[scale=0.12]{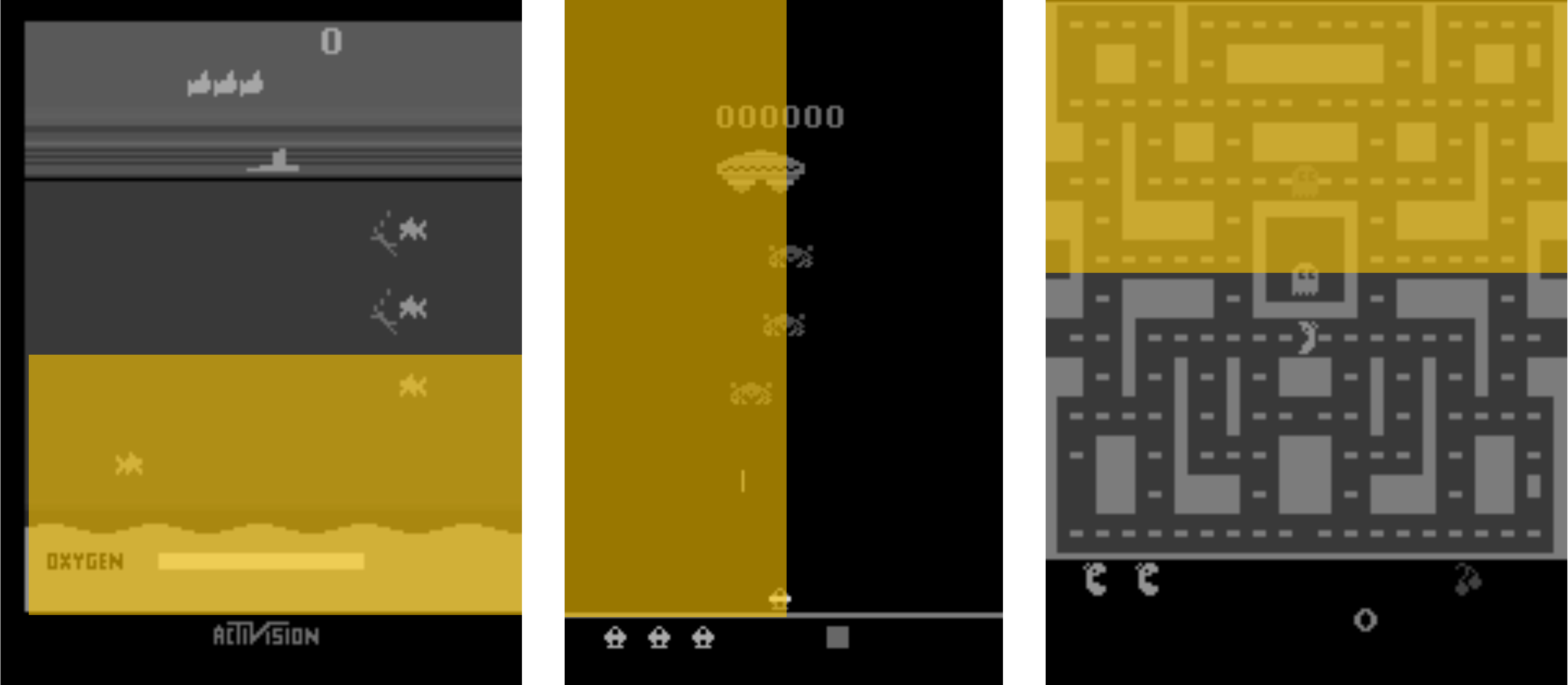}
\caption{Illustration of the restrictions placed on the rewards of the Atari 2600 environments in generalization experiments. The golden region represents the portion of the screen that the agent must be located in in order to receive reward. From left to right: Seaquest, Assault and Pacman.}
\label{fig:restricted_rewards}
\end{figure}

\begin{figure*}
\centering
\begin{subfigure}
\centering
\includegraphics[scale=0.3]{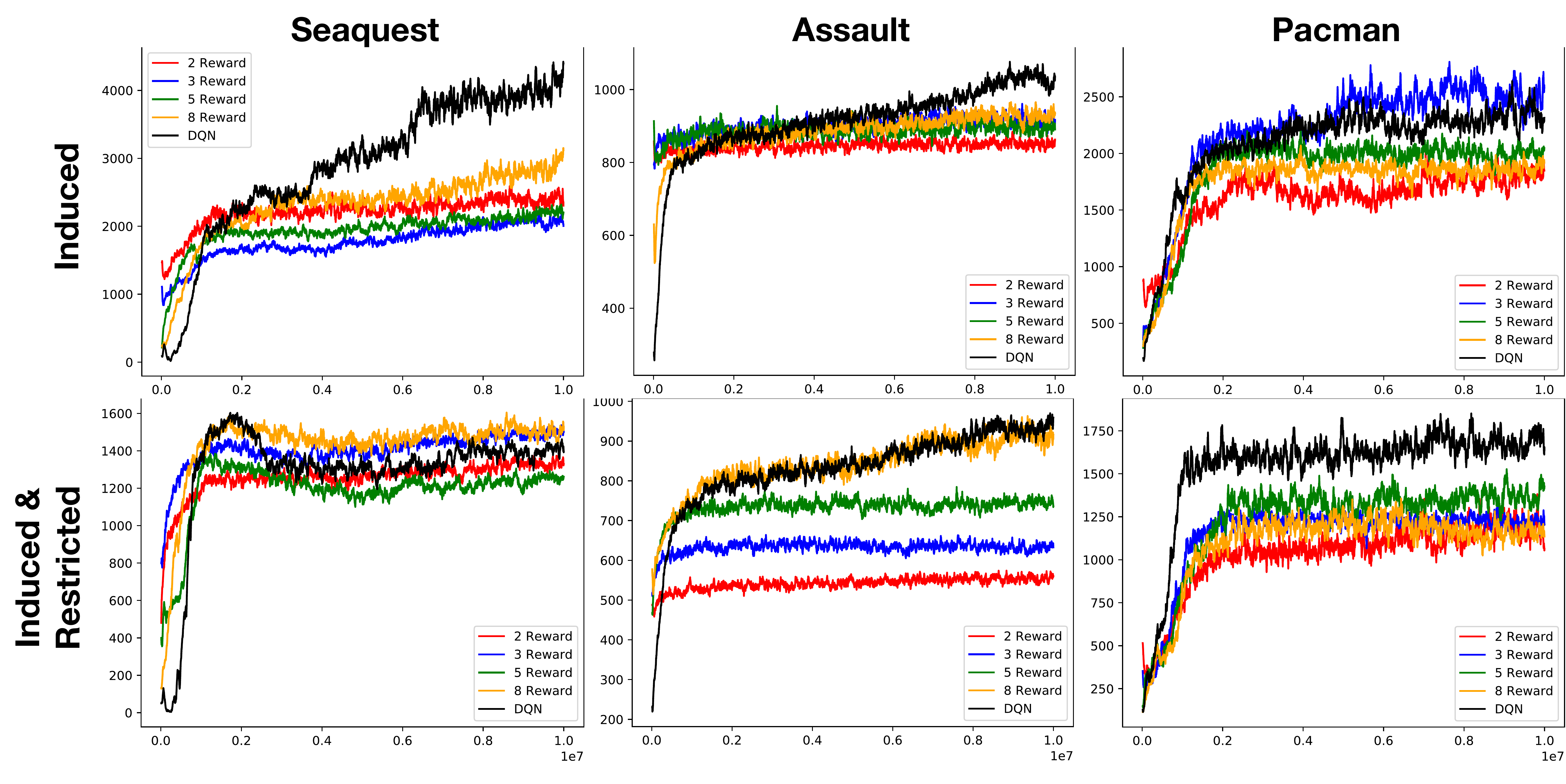}
\end{subfigure}
\caption{Each row depicts a comparison on various Atari 2600 games between learning control with our method's learned policies (labeled: 2 Reward, 3 Reward, 5 Reward and 8 Reward) in the induced MDP setting and the original original MDP (labeled: DQN). The top row depicts this comparison on the original Atari 2600 games. The bottom row shows the performance of our method's policies applied on modified versions of the games with restricted reward function (see Figure \ref{fig:restricted_rewards}).   
} \label{fig:meta_controller_results}
\end{figure*}

\section{Conclusion and Future Work}
In this work we presented and explored a novel formulation for additively-decomposing rewards into independently obtainable components. With empirical investigations, we showed that our disentanglement score is predictive of qualitatively interesting decompositions, that our algorithm is able to learn independently obtainable rewards for 3 Atari games better than a recent approach to disentangling states and policies, and that the policies optimal with respect to the learned rewards are useful in that they also obtain portions of the environment reward. With theoretical investigations, we showed that when our rewards are independently obtainable their optimal policies occupy non-overlapping states and that our gradient-based method for finding the reward decomposition yields saturated rewards in which each state's environment reward is allocated entirely to one learned reward.

\section{Acknowledgements} This work was supported by a grant from Toyota Research Institute (TRI), and by a grant from DARPA’s L2M program. Any opinions, findings, conclusions, or recommendations expressed here are those of the authors and do not necessarily reflect the views of the sponsors.
\clearpage 

\nocite{langley00}

\newpage
\bibliography{iclr2019_conference}
\bibliographystyle{icml2019}

\clearpage 
\appendix
\section{Algorithm Description} \label{app:algorithm}
\begin{algorithm}
Randomly initialize $\theta$ and $\phi_{1:n}$ for $R^\theta_{1:n}$ and $Q^{\phi_{1}},\ldots, Q^{\phi_{n}}$ respectively\;\\ 
Initialize empty replay buffer $\mathcal{D}$ \; \\
Randomly choose $I \in [n]$. \\
\While{True}{
Choose $\epsilon$-greedy action $a$ using $Q^{\phi_I}$. \\
Take action $a$ and observe state $s^\prime$ and reward $r$. \\
If state is terminal, re-sample $I$. \\
$\mathcal{D} = \mathcal{D} \cup (s, a, r, s^\prime)$\\
\For{$i \in [n]$}{
Sample minibatch $(s_t, a_t, r_t, s^\prime_t)_{t=1}^N$ from $\mathcal{D}$. \\
Replace each $r_t$ with $R_i(s^\prime_t)$. \\
Update $\phi_j$ using the Q-Learning algorithm.
}
Sample minibatch of starting states $(s_t)_{t=1}^N$ from $\mathcal{D}$. \\ 
\For{$i, j \in [n], k \in [N] $}{
Sample $T$-step trajectory $\tau_j^{i}$ starting from $s_k$ following $Q^{\phi_i}$. \\
Compute trajectory values under each $R_j^\theta$. 
}
Approximate $\nabla_\theta J_\text{disentangled}(\theta)$ as in Equation~\ref{eqn:gradient} using the sampled trajectory values. \\
$\theta \leftarrow \theta  + \eta \nabla_\theta J_\text{disentangled}(\theta)$. 
}
\caption{Learning Independently-Obtainable Rewards}\label{alg:low_level}
\end{algorithm} 

\section{Experimental Details}\label{app:experiment_details}

\paragraph{Data Preprocessing}
For Atari 2600 games we concatenate the last 3 frames before passing them to our networks. We do not do this for our gridworld experiments. We clip all Atari 2600 rewards between 0 and 1 and do not mark any states as terminal.

\paragraph{General}
All networks are trained for $10^7$ steps. All network training is done with a batch-size of 32. A single replay buffer is shared across DQN instances with a capacity of $10^5$. We annealed our epsilon greedy behavior policy from $\epsilon = 1$ to $\epsilon=0.01$ over the course of $10^6$ time-steps. We updated our Policy networks every $4$ time-steps and our Reward Decomposition networks every $20$ time-steps. For all experiments we use a discount value of $\gamma = 0.99$. 

\paragraph{Reward Decomposition Network}
The Reward Decomposition (RD) network was trained using the Adam optimizer with a learning rate of $5\cdot10^{-5}$. 

Our RD network maps states to vectors of $n$ decomposed rewards and consists of 3 convolutional layers followed by 2 fully connected layers. The convolutional layers have filter sizes: 8, 4, 3; numbers of filters: 32, 64, 64; strides: 4, 2, 1 and activations: relu, relu, relu. The fully connected layers have widths: 128, $n$ and activations: relu, softmax.

When approximating value functions as in Equation~\ref{eqn:gradient}, we roll out each policy for $10$ steps. We do not compute gradients through the weighting functions $\alpha_{i,j}$.

\paragraph{Policy Networks (DQNs)}
Our Policy networks are represented as dueling DQNs adapted from the OpenAI baselines repository.\footnote{https://github.com/openai/baselines} These networks each use a learning rate of $10^{-4}$, a target update frequency of $10^4$ time-steps. Each network consists of 3 convolutional layers followed by 2 fully connected layers. The convolutional layers have filter sizes: 8, 4, 3; numbers of filters: 32, 64, 64; strides: 4, 2, 1 and activations: relu, relu, relu. The fully connected layers have widths: 256, $|\mathcal{A}|$ and activations: relu, unit. 

\paragraph{ICF}
Our network for replicating ICF was adapted from the authors' git respository,\footnote{https://github.com/bengioe/implementations/tree/master/DL/ICF\_simple} leaving the majority of settings unchanged. The network, when applied to an image with $n$ channels, consists of an autoencoder with an auxiliary branch that outputs the latent state representation and policy distributions. The encoder consists of 3 convolutional layers with filter sizes: 3, 3, 3; numbers of filters: 32, 64, 64; strides: 2,2,2 and activations: relu, relu, relu. The decoder consists of 3 transpose-convolutional layers with filter sizes: 3,3,3; numbers of filters 64, 32, n; strides: 2,2,2 and activations: relu,relu, unit. The auxiliary network branch has a single fully-connected layer with width 32 and relu activation before splitting two branches for latent factors and policy distributions respectively. 

\section{Mathematical Results} \label{app:mathematical_results}
\paragraph{Relationship between $J_\text{indep}$ and $J_\text{nontriv}$}
We next illustrate that the $J_\text{indep}$ and $J_\text{nontriv}$ terms defined in Equations~\ref{eq:J_indep} and \ref{eq:J_nontriv} can be related to each other by a function that does not directly depend on the choice of reward decomposition, but on the behavior of the agents resulting from the decomposition. 

For a set of policies $\pi_1, \ldots, \pi_n$, we define this function as:
\begin{equation*}
\mathbb{V}(\pi_{1:n}) = \mathbb{E}_{s \sim \mu}\left[\sum_{i=1}^n V^{\pi_i}(s)\right]
\end{equation*}
which intuitively captures the ``total value'' obtained by a the policies $\pi_1, \ldots, \pi_n$.

\begin{lemma}\label{lemma:relationship}
If $\alpha_{i,j}(s) = C$ \ for all $s \in \mathcal{S}$ and $i,j \in [n]$ it follows that:
\begin{equation*}
J_\text{indep} + J_\text{nontriv} = C \cdot \mathbb{V}(\pi_{1:n}^*)
\end{equation*}
\end{lemma}

\begin{proof}
The statement can be shown by simple algebraic manipulations:
\begin{equation*}
\begin{aligned}
&J_\text{indep} + J_\text{nontriv} \\
&= \mathbb{E}_{s \sim \mu}\left[\sum_{i \neq j} \alpha_{i,j}(s) \subvalue{i}{\pi_j^*}(s) + \sum_{i=1}^n \alpha_{i,i}(s) \subvalue{i}{\pi_i^*}(s)\right] \\
&= C \cdot \mathbb{E}_{s \sim \mu}\left[\sum_{i=1}^n \sum_{j=1}^n \subvalue{i}{\pi_j^*}(s)\right] \\
&= C \cdot \mathbb{E}_{s \sim \mu} \left[\sum_{j=1}^n V^{\pi^*_j}(s)\right] = C \cdot \mathbb{V}(\pi^*_{1:n})
\end{aligned}
\end{equation*}
as needed.
\end{proof}

We can then define the ``sensitivity'' of an MDP's total value with respect to a change in policies as:
\begin{equation*}
S(\pi_{1:n}, \pi_{1:n}^\prime) = |\mathbb{V}(\pi_{1:n}) - \mathbb{V}(\pi_{1:n}^\prime)|.
\end{equation*}





\end{document}